\documentclass[twoside]{article}
\usepackage[accepted]{aistats2020}

\usepackage{algorithm}
\usepackage{algorithmicx}
\usepackage[noend]{algpseudocode}
\usepackage{amsmath}
\usepackage{amssymb}
\usepackage{amsthm}
\usepackage{bbm}
\usepackage{bm}
\usepackage{color}
\usepackage{dirtytalk}
\usepackage{dsfont}
\usepackage{enumerate}
\usepackage{graphicx}
\usepackage{mathtools}
\usepackage{natbib}
\usepackage{subfigure}
\usepackage{times}
\usepackage{xspace}

\usepackage[usenames,dvipsnames]{xcolor}
\usepackage[bookmarks=false]{hyperref}
\hypersetup{
  pdffitwindow=true,
  pdfstartview={FitH},
  pdfnewwindow=true,
  colorlinks,
  linktocpage=true,
  linkcolor=Green,
  urlcolor=Green,
  citecolor=Green
}
\usepackage[capitalize,noabbrev]{cleveref}

\Crefname{corollary}{Corollary}{Corollaries}
\Crefname{proposition}{Proposition}{Propositions}
\Crefname{theorem}{Theorem}{Theorems}
\Crefname{definition}{Definition}{Definitions}
\Crefname{assumption}{Assumption}{Assumptions}
\Crefname{example}{Example}{Examples}
\Crefname{remark}{Remark}{Remarks}
\Crefname{setting}{Setting}{Settings}
\Crefname{lemma}{Lemma}{Lemmas}

\usepackage[backgroundcolor=white]{todonotes}

\newcommand{\commentout}[1]{}
\newcommand{\junk}[1]{}

\usepackage{thmtools}
\declaretheorem[name=Theorem,refname={Theorem,Theorems},Refname={Theorem,Theorems}]{theorem}
\declaretheorem[name=Lemma,refname={Lemma,Lemmas},Refname={Lemma,Lemmas},sibling=theorem]{lemma}

\newcommand{\cD}{\mathcal{D}}
\newcommand{\cF}{\mathcal{F}}

\newcommand{\cX}{\mathcal{X}}

\newcommand{\realset}{\mathbb{R}}

\newcommand{\E}[1]{\mathbb{E} \left[#1\right]}
\newcommand{\condE}[2]{\mathbb{E} \left[#1 \,\middle|\, #2\right]}
\newcommand{\Et}[1]{\mathbb{E}_t \left[#1\right]}
\newcommand{\prob}[1]{\mathbb{P} \left(#1\right)}
\newcommand{\condprob}[2]{\mathbb{P} \left(#1 \,\middle|\, #2\right)}
\newcommand{\probt}[1]{\mathbb{P}_t \left(#1\right)}
\newcommand{\var}[1]{\mathrm{var} \left[#1\right]}

\newcommand{\normGt}[1]{\normw{#1}{G_t^{-1}}}

\newcommand{\abs}[1]{\left|#1\right|}

\newcommand{\I}[1]{\mathds{1} \! \left\{#1\right\}}

\newcommand{\normw}[2]{\|#1\|_{#2}}

\newcommand{\set}[1]{\left\{#1\right\}}

\newcommand{\T}{^\top}

\DeclareMathOperator*{\argmax}{arg\,max\,}
\DeclareMathOperator*{\argmin}{arg\,min\,}
\let\det\relax
\DeclareMathOperator{\det}{det}

\mathchardef\mhyphen="2D

\newcommand{\deepfl}{\ensuremath{\tt DeepFL}\xspace}
\newcommand{\deepfpl}{\ensuremath{\tt DeepFPL}\xspace}

\newcommand{\glmfpl}{\ensuremath{\tt GLM\mhyphen FPL}\xspace}
\newcommand{\glmtsl}{\ensuremath{\tt GLM\mhyphen TSL}\xspace}
\newcommand{\glmucb}{\ensuremath{\tt GLM\mhyphen UCB}\xspace}

\newcommand{\lints}{\ensuremath{\tt LinTS}\xspace}
\newcommand{\linucb}{\ensuremath{\tt LinUCB}\xspace}

\newcommand{\repglmfpl}{\ensuremath{\tt repGLM\mhyphen FPL}\xspace}

\newcommand{\ucbglm}{\ensuremath{\tt UCB\mhyphen GLM}\xspace}

\begin{document}

\twocolumn[
\aistatstitle{Randomized Exploration in Generalized Linear Bandits}
\aistatsauthor{
Branislav Kveton \And
Manzil Zaheer \And
Csaba Szepesv\'ari \And
Lihong Li}
\aistatsaddress{
Google Research \And
Google Research \And
DeepMind / University of Alberta \And
Google Research}
\aistatsauthor{
Mohammad Ghavamzadeh \And
Craig Boutilier}
\aistatsaddress{
Facebook AI Research \And
Google Research}]

\begin{abstract}
We study two randomized algorithms for generalized linear bandits. The first, \glmtsl, samples a generalized linear model (GLM) from the \emph{Laplace approximation} to the posterior distribution. The second, \glmfpl, fits a GLM to a \emph{randomly perturbed history} of past rewards. We analyze both algorithms and derive $\tilde{O}(d \sqrt{n \log K})$ upper bounds on their $n$-round regret, where $d$ is the number of features and $K$ is the number of arms. The former improves on prior work while the latter is the first for Gaussian noise perturbations in non-linear models. We empirically evaluate both \glmtsl and \glmfpl in logistic bandits, and apply \glmfpl to neural network bandits. Our work showcases the role of randomization, beyond posterior sampling, in exploration.
\end{abstract}


\section{Introduction}
\label{sec:introduction}

A \emph{multi-armed bandit} \citep{lai85asymptotically,auer02finitetime,lattimore19bandit} is an online learning problem where actions of the \emph{learning agent} are represented by \emph{arms}. The arms can be treatments in a clinical trial or ads on a website. After an arm is \emph{pulled}, the agent receives a \emph{stochastic reward}. The agent aims to maximize its expected cumulative reward. Since the agent does not know the mean rewards of the arms in advance, it faces the \emph{exploration-exploitation dilemma}: \emph{explore}, and learn more about the reward distributions of the arms; or \emph{exploit}, and pull the arm with the highest estimated reward thus far.

A \emph{generalized linear bandit} \citep{filippi10parametric,zhang16online,li17provably,jun17scalable} is a variant of the multi-armed bandit where the expected rewards of arms are modeled using a \emph{generalized linear model (GLM)} \citep{mccullagh89generalized}. Specifically, the expected reward is a known function $\mu$, such as a sigmoid, of the dot product of a known feature vector and an unknown parameter vector. In the earlier clinical example, the feature and parameter vectors could be treatment indicators and effects of individual treatments, respectively.

Most existing algorithms for generalized linear bandits are based on \emph{upper confidence bounds (UCBs)}. Motivated by the superior performance of randomized GLM algorithms \citep{chapelle11empirical,russo18tutorial}, we study two randomized algorithms for this class of problems, \glmtsl and \glmfpl. \glmtsl samples a GLM from the Laplace approximation to the posterior distribution. \glmfpl fits a GLM to a \emph{randomly perturbed history} of past rewards.

We analyze \glmtsl and \glmfpl, and prove that their $n$-round regret is $\tilde{O}(d \sqrt{n \log K})$, where $d$ is the number of features and $K$ is the number of arms. The regret bound of \glmtsl improves on the best prior regret bound \citep{abeille17linear} by a multiplicative factor of $\sqrt{d / \log K}$ in the finite arm setting and matches it in the infinite arm setting. The regret bound of \glmfpl is the first for Gaussian noise perturbations in non-linear models, although we derive it under an additional assumption on arm features.

We also evaluate \glmtsl and \glmfpl empirically. Both have a state-of-the-art performance in logistic bandits, the most important practical use case of GLM bandits. Just as importantly, the perturbation scheme in \glmfpl generalizes straightforwardly to complex reward models, such as a neural network. To demonstrate this, we apply \glmfpl to high-dimensional classification problems and show that it can learn complex neural network mappings from features to rewards. The simplicity of \glmfpl suggests that it may find broad application in the future.


\section{Setting}
\label{sec:setting}

\begin{algorithm}[t]
  \caption{General randomized exploration in generalized linear bandits.}
  \label{alg:randomized mle}
  \begin{algorithmic}[1]
    \State \textbf{Inputs}: Number of exploration rounds $\tau$
    \Statex 
    \For{$t = 1, \dots, n$}
      \If{$t > \tau$}
        \State $\tilde{\theta}_t \gets$ Randomized MLE on
        $\set{(X_\ell, Y_\ell)}_{\ell = 1}^{t - 1}$
        \label{alg:randomized:estimation}
        \State $I_t \gets \argmax_{i \in [K]} x_i\T \tilde{\theta}_t$
        \label{alg:randomized:pulled arm}
      \Else
        \State Choose $I_t$ based on $\set{X_\ell}_{\ell = 1}^{t - 1}$
        \label{alg:randomized:initialization}
      \EndIf
      \State Pull arm $I_t$ and get reward $Y_{I_t, t}$
      \State $X_t \gets x_{I_t}, \ Y_t \gets Y_{I_t, t}$
    \EndFor
  \end{algorithmic}
\end{algorithm}

We adopt the following notation. The set $\set{1, \dots, n}$ is denoted by $[n]$. All vectors are column vectors. For any \emph{positive semi-definite (PSD)} matrix $M$, $\lambda_{\min}(M) \geq 0$ is the minimum eigenvalue of $M$. For any $n \times n$ PSD matrices $M_1$ and $M_2$, $M_1 \preceq M_2$ if and only if $x\T M_1 x \leq x\T M_2 x$ for all $x \in \realset^d$. We let $\normw{x}{M} = \sqrt{x\T M x}$ and $\mathrm{Ber}(p)$ be the Bernoulli distribution with mean $p$. The indicator that event $E$ occurs is $\I{E}$. We use $\tilde{O}$ for the big-O notation up to logarithmic factors in horizon $n$.

A \emph{generalized linear model (GLM)} is a probabilistic model where observation $Y$ given feature vector $x \in \realset^d$ has an exponential-family distribution with mean $\mu(x\T \theta)$, where $\mu$ is the \emph{mean function} and $\theta \in \realset^d$ are model parameters \citep{mccullagh89generalized}. Let $\cD = \set{(x_\ell, y_\ell)}_{\ell = 1}^n$ be a set of $n$ observations, where $x_\ell \in \realset^d$ and $y_\ell \in \realset$. The \emph{negative log likelihood} of $\cD$ under model parameters $\theta$ is
\begin{align*}
  L(\cD; \theta)
  = \sum_{\ell = 1}^{\abs{\cD}} b(x_\ell\T \theta) -
  y_\ell x_\ell\T \theta - c(y_\ell)\,,
\end{align*}
where $c$ is a real function, and $b$ is twice continuously differentiable and its derivative is the mean function, $\dot{b} = \mu$. The \emph{gradient} and \emph{Hessian} of $L(\cD; \theta)$ with respect to $\theta$ are
\begin{align}
  \nabla L(\cD; \theta)
  & = \sum_{\ell = 1}^{\abs{\cD}} (\mu(x_\ell\T \theta) - y_\ell) x_\ell\,,
  \label{eq:gradient} \\
  \nabla^2 L(\cD; \theta)
  & = \sum_{\ell = 1}^{\abs{\cD}} \dot{\mu}(x_\ell\T \theta) x_\ell x_\ell\T\,,
  \label{eq:hessian}
\end{align}
where $\dot{\mu}$ denotes the derivative of $\mu$. The mean function $\mu$ is increasing and therefore its derivative $\dot{\mu}$ is positive. The \emph{maximum likelihood estimate (MLE)} of model parameters is a vector $\theta \in \realset^d$ such that $\nabla L(\cD; \theta) = \mathbf{0}$.

A \emph{stochastic GLM bandit} \citep{filippi10parametric} is an online learning problem where the rewards of arms are generated by some underlying GLM. Let $K$ be the number of arms, $x_i \in \realset^d$ be the \emph{feature vector} of arm $i \in [K]$, and $\theta_\ast \in \realset^d$ be an unknown \emph{parameter vector}. The \emph{reward} of arm $i$ in round $t \in [n]$, $Y_{i, t}$, is drawn i.i.d.\ from a distribution with mean $\mu_i = \mu(x_i\T \theta_\ast)$. We assume that $\eta_{i, t} = Y_{i, t} - \mu_i$ is $\sigma^2$-sub-Gaussian, that is
\begin{align*}
  \E{\exp[\lambda \eta_{i, t}]} \leq \exp[\lambda^2 \sigma^2 / 2]
\end{align*}
holds for all arms $i$, rounds $t$, and $\lambda \geq 0$. In round $t$, the agent \emph{pulls} arm $I_t \in [K]$ and observes its reward $Y_{I_t, t}$. The goal of the agent is to maximize its \emph{expected cumulative reward} in $n$ rounds. To simplify notation, we denote the feature vector of arm $I_t$ by $X_t = x_{I_t}$ and its stochastic reward by $Y_t = Y_{I_t, t}$.

Without loss of generality, we assume that arm $1$ is the \emph{unique optimal arm}, that is $\mu_1 > \max_{i > 1} \mu_i$. Let $\Delta_i = \mu_1 - \mu_i$ be the \emph{suboptimality gap} of arm $i$. Maximization of the expected cumulative reward over $n$ rounds is equivalent to minimizing the \emph{expected $n$-round regret}, which is defined as
\begin{align}
  R(n)
  = \sum_{i = 2}^K \Delta_i \E{\sum_{t = 1}^n \I{I_t = i}}\,.
  \label{eq:regret}
\end{align}


\section{Algorithms}
\label{sec:algorithms}

Our GLM bandit algorithms follow the template in \cref{alg:randomized mle}. They \emph{explore} initially in $\tau$ rounds, so that the estimated parameters in subsequent rounds have \say{good} properties. The exploration strategy is detailed in \cref{sec:discussion}. After the initial exploration, they act greedily with respect to \emph{randomized parameter vectors} $\tilde{\theta}_t$. Specifically, they pull arm $I_t = \argmax_{i \in [K]} x_i\T \tilde{\theta}_t$ in round $t$. If this maximum is not unique, any tie breaking can be used.

\subsection{Algorithm \glmtsl}
\label{sec:algorithm glm-tsl}

We study two algorithms. The first algorithm, \glmtsl, is a variant of \emph{Thompson sampling} \citep{thompson33likelihood} where the posterior of $\theta_\ast$ is approximated by its \emph{Laplace approximation}. The randomized parameter vector is sampled from the Laplace approximation
\begin{align}
  \tilde{\theta}_t
  \sim \mathcal{N}(\bar{\theta}_t, a^2 H_t^{-1})\,,
  \label{eq:glm-tsl}
\end{align}
where
\begin{align}
  \begin{split}
    \bar{\theta}_t
    & = \argmin_{\theta \in \realset^d}
    L(\set{(X_\ell, Y_\ell)}_{\ell = 1}^{t - 1}; \theta)\,, \\
    H_t
    & = \sum_{\ell = 1}^{t - 1} \dot{\mu}(X_\ell\T \bar{\theta}_t) X_\ell X_\ell\T\,,
  \end{split}
  \label{eq:mle}
\end{align}
and $a > 0$ is a tunable parameter. \citet{chapelle11empirical} and \citet{russo18tutorial} evaluated \glmtsl empirically. In addition, \citet{abeille17linear} proved that \glmtsl has $\tilde{O}(d^\frac{3}{2} \sqrt{n})$ regret in the infinite arm setting. We prove that \glmtsl has $\tilde{O}(d \sqrt{n \log K})$ regret when the number of arms is $K$.

\subsection{Algorithm \glmfpl}
\label{sec:algorithm glm-fpl}

We also propose a \emph{follow-the-perturbed-leader (FPL)} algorithm, \glmfpl. In \glmfpl, the randomized parameter vector is the MLE from past $t - 1$ rewards \emph{perturbed with Gaussian noise},
\begin{align}
  \tilde{\theta}_t
  = \argmin_{\theta \in \realset^d}
  L(\set{(X_\ell, Y_\ell + Z_\ell)}_{\ell = 1}^{t - 1}; \theta)\,,
  \label{eq:glm-fpl}
\end{align}
where $Z_\ell \sim \mathcal{N}(0, a^2)$ are normal random variables that are resampled in each round, independently of each other and the history, and $a > 0$ is a tunable parameter. Surprisingly, this perturbation does not change the parameter estimation problem. In particular, it only shifts the gradient of the log likelihood in \eqref{eq:gradient} by $Z_\ell X_\ell$ and the Hessian in \eqref{eq:hessian} remains positive semi-definite. In this work, we show that \glmfpl has $\tilde{O}(d \sqrt{n \log K})$ regret when the number of arms is $K$, under an additional assumption on arm features.

The design of \glmfpl is motivated by the equivalence of posterior sampling and perturbations by Gaussian noise in linear models \citep{lu17ensemble}, when the prior of $\theta_\ast$ and rewards are Gaussian. In GLMs, these two are not equivalent. Thus \glmtsl and \glmfpl are different algorithms. \glmfpl can be also viewed as an instance of randomized least-squares value iteration \citep{osband16generalization} applied to bandits. The specific instance in this work, additive Gaussian noise in a GLM, is novel. Finally, we note that the perturbation in \eqref{eq:glm-fpl} can be directly applied to more complex models, such as neural networks (\cref{sec:experiments}). This is arguably its most attractive property.

\subsection{Computationally-Efficient Implementations}
\label{sec:computationally-efficient implementations}

The MLEs in \eqref{eq:glm-tsl} and \eqref{eq:glm-fpl} can be computed by \emph{iteratively reweighted least squares (IRLS)} \citep{wolke88iteratively}, which uses Newton's method. Roughly speaking, each step of IRLS multiplies the inverse of \eqref{eq:hessian} and \eqref{eq:gradient}. If \eqref{eq:hessian} and \eqref{eq:gradient} can be expressed independently of round $t$, the computational cost of an IRLS step does not increase with $t$. This is viable for any set of feature vectors $\cX$ using
\begin{align*}
  \textstyle
  \sum_{x \in \cX} (N_x \mu(x^T \theta) - Y_x) x\,, \quad
  \sum_{x \in \cX} N_x \dot{\mu}(x^T \theta) x x^T\,,
\end{align*}
where $N_x$ is the number of times that $x$ appears in history $\cD$, and $Y_x$ is the sum of its rewards. Both $N_x$ and $Y_x$ can be updated incrementally. Finally, adding $\mathcal{N}(0, a^2)$ noise to each reward in \eqref{eq:glm-fpl} is equivalent to adding $\mathcal{N}(0, N_x a^2)$ noise to each $Y_x$ above.

The pulled arm in line \ref{alg:randomized:pulled arm} of \cref{alg:randomized mle} can be computed efficiently even when the arm space is infinite, such as an intersection of half spaces. This is true of prior GLM bandit algorithms (\cref{sec:related work}). The MLE in line \ref{alg:randomized:estimation} cannot be computed efficiently in general, independently of round $t$, as in all prior algorithms except that of \citet{jun17scalable}. We study one approximation empirically in \cref{sec:deep bandit experiments}.


\section{Analysis}
\label{sec:analysis}

Our analysis is organized as follows. In \cref{sec:technical challenges}, we review technical challenges that arise in analyzing GLM bandits and their solutions. In \cref{sec:outline}, we outline our analysis. In \cref{sec:glm-tsl analysis,sec:glm-fpl analysis}, we prove regret bounds for \glmtsl and \glmfpl. We discuss them in \cref{sec:discussion}.

\subsection{Technical Challenges}
\label{sec:technical challenges}

One challenge in analyzing GLMs is that they do not have closed-form solutions. Nevertheless, their solutions can be expressed using the gradient and Hessian of the log likelihood (\cref{sec:setting}). This is the key idea in the analyses of GLM bandits \citep{filippi10parametric,li17provably} and we present it below.

\begin{lemma}
\label{lem:workhorse} 
Let $\cD_1 = \set{(x_\ell, y_{\ell, 1})}_{\ell = 1}^n$ be a set of $n$ observations and $\cD_2 = \set{(x_\ell, y_{\ell, 2})}_{\ell = 1}^n$ have the same features as $\cD_1$. Let $\theta_1$ be the minimizer of $L(\cD_1; \theta)$ and $\theta_2$ be the minimizer of $L(\cD_2; \theta)$. Then
\begin{align*}
  \sum_{\ell = 1}^n (y_{\ell, 2} - y_{\ell, 1}) x_\ell
  = \nabla^2 L(\cD_1; \theta') (\theta_2 - \theta_1)\,,
\end{align*}
where $\theta' = \alpha \theta_1 + (1 - \alpha) \theta_2$ for some $\alpha \in [0, 1]$.
\end{lemma}
\begin{proof}
By the definition of the gradient in \eqref{eq:gradient},
\begin{align*}
  \nabla L(\cD_1; \theta) - \nabla L(\cD_2; \theta)
  = \sum_{\ell = 1}^n (y_{\ell, 2} - y_{\ell, 1}) x_\ell
\end{align*}
holds for any $\theta$. Moreover, from the definitions of $\theta_1$ and $\theta_2$, $\nabla L(\cD_1; \theta_1) = \nabla L(\cD_2; \theta_2) = \mathbf{0}$. Now we apply these identities and obtain
\begin{align*}
  \sum_{\ell = 1}^n (y_{\ell, 2} - y_{\ell, 1}) x_\ell
  & = \nabla L(\cD_1; \theta_2) - \nabla L(\cD_2; \theta_2) \\
  & = \nabla L(\cD_1; \theta_2) - \nabla L(\cD_1; \theta_1) \\
  & = \nabla^2 L(\cD_1; \theta') (\theta_2 - \theta_1)\,.
\end{align*}
where $\theta'$ is defined in the claim.
\end{proof}

Another challenge is $\dot{\mu}(x_\ell\T \theta)$ in \eqref{eq:hessian}. To apply ideas from linear bandit analyses, it must be eliminated. We do so as follows. Let $G = \sum_{\ell = 1}^{\abs{\cD}} x_\ell x_\ell\T$ be an \emph{unweighted Hessian} with the same features as \eqref{eq:hessian}. Let $c_{\min} \leq \dot{\mu}(x_\ell\T \theta) \leq c_{\max}$ for some $c_{\min}$ and $c_{\max}$, and for all $\ell \in [\abs{\cD}]$. Then from the definition of \eqref{eq:hessian}, $c_{\min} G \preceq \nabla^2 L(\cD; \theta) \preceq c_{\max} G$ and $c_{\min}^{-1} G^{-1} \succeq (\nabla^2 L(\cD; \theta))^{-1} \succeq c_{\max}^{-1} G^{-1}$. Because of this, the derivatives of $\mu$ must be controlled.

To control the derivatives of $\mu$ at $\bar{\theta}_t$ and $\tilde{\theta}_t$ (\cref{sec:algorithms}), we initially explore so that $\bar{\theta}_t$ and $\tilde{\theta}_t$ are in the unit ball centered at $\theta_\ast$ with a high probability. This gives rise to
\begin{align*}
  \textstyle
  \dot{\mu}_{\min}
  = \min_{\normw{x}{2} \leq 1, \, \normw{\theta - \theta_\ast}{2} \leq 1}
  \dot{\mu}(x\T \theta)
\end{align*}
in our regret bounds, the \emph{minimum derivative of $\mu$} in the unit ball centered at $\theta_\ast$. This trick \citep{li17provably} requires that $\normw{x_i}{2} \leq 1$ for all arms $i$, and we assume this in our analysis. We define the \emph{maximum derivative of $\mu$} as
\begin{align*}
  \textstyle
  \dot{\mu}_{\max}
  = \max_{\normw{x}{2} \leq 1, \, \theta \in \realset^d} \dot{\mu}(x\T \theta)\,.
\end{align*}
This factor is typically easy to control. In logistic regression, for instance, $\dot{\mu}_{\max} = 1 / 4$.

\subsection{Outline of Our Analyses}
\label{sec:outline}

Let $\theta_\ast$ be the unknown parameter vector, $\bar{\theta}_t$ be its MLE in round $t$, and $\tilde{\theta}_t$ be the randomized MLE in round $t$. At a high level, we bound the regret under assumptions that $\bar{\theta}_t \to \theta_\ast$, $\tilde{\theta}_t \to \bar{\theta}_t$, and $\tilde{\theta}_t$ is optimistic. We show that the corresponding favorable conditions hold with a high probability and define the corresponding events below.

Let $\cF_t = \sigma(I_1, \dots, I_t, Y_1, \dots, Y_t)$ be the $\sigma$-algebra generated by the pulled arms and their rewards by the end of round $t \in [n]$. We let $\cF_0 = \set{\emptyset, \Omega}$, where $\Omega$ is the sample space of the probability space that holds all random variables. Then $(\cF_t)_t$ is a filtration. Let
\begin{align*}
  \probt{\cdot}
  = \condprob{\cdot}{\cF_{t - 1}}\,, \quad
  \Et{\cdot}
  = \condE{\cdot}{\cF_{t - 1}}\,,
\end{align*}
be the conditional probability and expectation, given the history at the beginning of round $t$, $\cF_{t - 1}$, respectively. Let $G_t = \sum_{\ell = 1}^{t - 1} X_\ell X_\ell\T$ be the \emph{unweighted Hessian} in round $t$ and $\Delta_{\max} = \max_{i \in [K]} \Delta_i$ be the maximum regret.

To argue that $\bar{\theta}_t \to \theta_\ast$, we define
\begin{align}
  E_{1, t}
  = \set{\forall i \in [K]:
  \abs{x_i\T \bar{\theta}_t - x_i\T \theta_\ast} \leq c_1 \normGt{x_i}}\,,
  \label{eq:theta bar is close}
\end{align}
the event that $x_i\T \bar{\theta}_t$ and $x_i\T \theta_\ast$ are \say{close} for all arms $i$ in round $t$, where $c_1 > 0$ is tuned later such that event $E_{1, t}$ is likely. Specifically, let $\bar{E}_{1, t}$ be the complement of $E_{1, t}$. Then we set $c_1$ such that $\prob{\bar{E}_{1, t}} = O(1 / n)$.

The upper bound on $\prob{\bar{E}_{1, t}}$ is motivated by Lemma 3 in \citet{li17provably}. We reprove the lemma since it contains a subtle error. In particular, the proof that $\normw{\bar{\theta}_t - \theta_\ast}{2} \leq 1$ holds with a high probability assumes that the agent does not act adaptively up to round $t$, which it clearly \emph{does} for any $t > \tau$.

To argue that $\tilde{\theta}_t \to \bar{\theta}_t$, we define
\begin{align}
  E_{2, t}
  = \set{\forall i \in [K]:
  \abs{x_i\T \tilde{\theta}_t - x_i\T \bar{\theta}_t} \leq c_2 \normGt{x_i}}\,,
  \label{eq:theta tilde is close}
\end{align}
the event that $x_i\T \tilde{\theta}_t$ and $x_i\T \bar{\theta}_t$ are \say{close} for all arms $i$ in round $t$, where $c_2 > 0$ is tuned later such that event $E_{2, t}$ is likely given any past. Specifically, let $\bar{E}_{2, t}$ be the complement of $E_{2, t}$. Then we set $c_2$ such that $\probt{\bar{E}_{2, t}} = O(1 / n)$. This part of the analysis relies on the properties of our perturbations and is novel.

Finally, to argue that $\tilde{\theta}_t$ is sufficiently optimistic given any past, we define event
\begin{align}
  E_{3, t}
  = \set{x_1\T \tilde{\theta}_t - x_1\T \bar{\theta}_t > c_1 \normGt{x_1}}.
  \label{eq:theta tilde is optimistic}
\end{align}
To obtain $\probt{E_{3, t}} = O(1)$, we set parameter $a$ in \eqref{eq:glm-tsl} and \eqref{eq:glm-fpl}. This part of the analysis relies on the properties of our perturbations and is novel.

Our analysis is sufficiently general, so that it can be used to analyze different randomized algorithms. To show this, we use it to analyze both \glmtsl and \glmfpl. The central part of the analysis is an upper bound on the expected per-round regret of any randomized algorithm whose perturbed solution in round $t$ is a function of its history. The corresponding lemma is stated below.

\begin{restatable}[]{lemma}{perroundregret}
\label{lem:per-round regret} Let $p_2 \geq \probt{\bar{E}_{2, t}}$, $p_3 \leq \probt{E_{3, t}}$, and $p_3 > p_2$. Then on event $E_{1, t}$,
\begin{align*}
  \Et{\Delta_{I_t}}
  \leq {} & \dot{\mu}_{\max} (c_1 + c_2)
  \left(1 + \frac{2}{p_3 - p_2}\right) \times {} \\
  & \Et{\normGt{x_{I_t}}} + \Delta_{\max} \, p_2\,.
\end{align*}
\end{restatable}

The hardest part in the analyses of \glmtsl and \glmfpl is to bound $p_2$ and $p_3$ in \cref{lem:per-round regret}.

\subsection{Analysis of \glmtsl}
\label{sec:glm-tsl analysis}

Now we are ready to analyze \glmtsl and \glmfpl. The regret bound of \glmtsl is stated below.

\begin{restatable}[]{theorem}{glmtslregretbound}
\label{thm:glm-tsl upper bound} The $n$-round regret of \glmtsl is bounded as
\begin{align*}
  R(n)
  \leq {} & \dot{\mu}_{\max} (c_1 + c_2)
  \left(1 + \frac{2}{0.15 - 1 / n}\right) \times {} \\
  & \sqrt{2 d n \log(2 n / d)} + (\tau + 3) \Delta_{\max}\,,
\end{align*}
where
\begin{align*}
  a
  & = c_1 \sqrt{\dot{\mu}_{\max}}\,, \\
  c_1
  & = \sigma \dot{\mu}_{\min}^{-1}
  \sqrt{d \log(n / d) + 2 \log n}\,, \\
  c_2
  & = c_1 \sqrt{2 \dot{\mu}_{\min}^{-1} \, \dot{\mu}_{\max} \log(K n)}\,,
\end{align*}
and the number of exploration rounds $\tau$ satisfies
\begin{align*}
  \lambda_{\min}(G_\tau)
  \geq \max \set{\sigma^2 \dot{\mu}_{\min}^{-2} (d \log(n / d) + 2 \log n), \, 1}\,.
\end{align*}
\end{restatable}
\begin{proof}
The claim is proved in \cref{sec:regret bounds}.

The proof has three key steps. First, we bound the probability of event $\bar{E}_{1, t}$ from above (\cref{lem:theta bar concentration} in \cref{sec:technical lemmas}). Second, we choose parameter $a$ such that the probabilities of events $\bar{E}_{2, t}$ and $E_{3, t}$ are bounded for any history $\cF_{t - 1}$ (\cref{lem:glm-tsl}). Finally, we set the number of initial exploration rounds $\tau$ such that $\normw{\bar{\theta}_t - \theta_\ast}{2} \leq 1$ is likely in any round $t \geq \tau$ (\cref{lem:glm-tsl unit ball} in \cref{sec:technical lemmas}).
\end{proof}

The above regret bound is $\tilde{O}(d \sqrt{n \log K})$. We derive the key concentration and anti-concentration lemma below.

\begin{lemma}
\label{lem:glm-tsl} Let
\begin{align*}
  a
  = c_1 \sqrt{\dot{\mu}_{\max}}\,, \quad
  c_2
  = c_1 \sqrt{2 \dot{\mu}_{\min}^{-1} \, \dot{\mu}_{\max} \log(K n)}\,.
\end{align*}
Let $E = \set{\normw{\bar{\theta}_t - \theta_\ast}{2} \leq 1}$. Then $\probt{\bar{E}_{2, t}} \leq 1 / n$ holds on event $E$ and $\probt{E_{3, t}} \geq 0.15$.
\end{lemma}
\begin{proof}
By the design of \glmtsl in \eqref{eq:glm-tsl},
\begin{align*}
  x\T \tilde{\theta}_t - x\T \bar{\theta}_t
  \sim \mathcal{N}(0, a^2 \normw{x}{H_t^{-1}}^2)
\end{align*}
for any vector $x \in \realset^d$, where matrix $H_t$ is defined in \eqref{eq:mle}. Let $U = x\T \tilde{\theta}_t - x\T \bar{\theta}_t$. Because $U \sim \mathcal{N}(0, a^2 \normw{x}{H_t^{-1}}^2)$ is a normal random variable, we have that
\begin{align*}
  \probt{U \geq a \normw{x}{H_t^{-1}}}
  & \geq 0.15\,, \\
  \probt{U \geq c \normw{x}{H_t^{-1}}}
  & \leq \exp\left[- \frac{c^2}{2 a^2}\right]\,,
\end{align*}
for any $c > 0$.

Now note that $H_t \preceq \dot{\mu}_{\max} G_t$. As a result,
\begin{align*}
  0.15
  & \leq \probt{U \geq a \normw{x}{H_t^{-1}}} \\
  & \leq \probt{U \geq a \sqrt{\dot{\mu}_{\max}^{-1}} \normGt{x}}\,.
\end{align*}
For $a = c_1 \sqrt{\dot{\mu}_{\max}}$ and $x = x_1$, we get that event $E_{3, t}$ in \eqref{eq:theta tilde is optimistic} occurs with probability at least $0.15$.

Moreover, $H_t \succeq \dot{\mu}_{\min} G_t$ on event $E$, which yields
\begin{align*}
  \exp\left[- \frac{c^2}{2 a^2}\right]
  & \geq \probt{U \geq c \normw{x}{H_t^{-1}}} \\
  & \geq \probt{U \geq c \sqrt{\dot{\mu}_{\min}^{-1}} \normGt{x}}\,.
\end{align*}
For $c = a \sqrt{2 \log(K n)}$, $x = x_i$, and by the union bound over all $K$ arms, we get that event $\bar{E}_{2, t}$ in \eqref{eq:theta tilde is close} occurs with probability at most $1 / n$.
\end{proof}

\subsection{Analysis of \glmfpl}
\label{sec:glm-fpl analysis}

The regret bound of \glmfpl is stated below. The analysis assumes that all feature vectors $x_i$ have at most one non-zero entry. This assumption is discussed in \cref{sec:discussion}.

\begin{restatable}[]{theorem}{glmfplregretbound}
\label{thm:glm-fpl upper bound} The $n$-round regret of \glmfpl is bounded as
\begin{align*}
  R(n)
  \leq {} & \dot{\mu}_{\max} (c_1 + c_2)
  \left(1 + \frac{2}{0.15 - 2 / n}\right) \times {} \\
  & \sqrt{2 d n \log(2 n / d)} + (\tau + 4) \Delta_{\max}\,,
\end{align*}
where
\begin{align*}
  a
  & = c_1 \dot{\mu}_{\max}\,, \\
  c_1
  & = \sigma \dot{\mu}_{\min}^{-1}
  \sqrt{d \log(n / d) + 2 \log n}\,, \\
  c_2
  & = c_1 \dot{\mu}_{\min}^{-1} \, \dot{\mu}_{\max} \sqrt{2 \log(K n)}\,,
\end{align*}
and the number of exploration rounds $\tau$ satisfies
\begin{align*}
  \lambda_{\min}(G_\tau)
  \geq \max \{& 4 \sigma^2 \dot{\mu}_{\min}^{-2} (d \log(n / d) + 2 \log n), \\
  & 8 a^2 \dot{\mu}_{\min}^{-2} \log n, \, 1\}\,.
\end{align*}
\end{restatable}
\begin{proof}
The claim is proved in \cref{sec:regret bounds}.

The proof has three key steps. First, we bound the probability of event $\bar{E}_{1, t}$ from above (\cref{lem:theta bar concentration} in \cref{sec:technical lemmas}). Second, we choose parameter $a$ such that the probabilities of events $\bar{E}_{2, t}$ and $E_{3, t}$ are bounded for any history $\cF_{t - 1}$ (\cref{lem:glm-fpl}). Finally, we set the number of initial exploration rounds $\tau$ such that $\normw{\bar{\theta}_t - \theta_\ast}{2} \leq 1 / 2$ is likely and $\normw{\tilde{\theta}_t - \theta_\ast}{2} \leq 1$ is conditionally likely given $\cF_{t - 1}$, in any round $t \geq \tau$ (\cref{lem:glm-fpl unit ball} in \cref{sec:technical lemmas}).
\end{proof}

The above regret bound is also $\tilde{O}(d \sqrt{n \log K})$. The key concentration and anti-concentration lemma follows.

\begin{lemma}
\label{lem:glm-fpl} Let
\begin{align*}
  a
  = c_1 \dot{\mu}_{\max}\,, \quad
  c_2
  = c_1 \dot{\mu}_{\min}^{-1} \, \dot{\mu}_{\max} \sqrt{2 \log(K n)}\,.
\end{align*}
Let $E = \set{\normw{\bar{\theta}_t - \theta_\ast}{2} \leq 1 / 2}$, $E' = \{\normw{\tilde{\theta}_t - \theta_\ast}{2} \leq 1\}$, and $\probt{\bar{E}'} \leq 1 / n$ on event $E$. Then $\probt{\bar{E}_{2, t}} \leq 2 / n$ on event $E$ and $\probt{E_{3, t}} \geq 0.15$.
\end{lemma}
\begin{proof}
Fix any history $\cF_{t - 1}$. By \cref{lem:workhorse}, where $\cD_1 = \set{(X_\ell, Y_\ell)}_{\ell = 1}^{t - 1}$ and $\cD_2 = \set{(X_\ell, Y_\ell + Z_\ell)}_{\ell = 1}^{t - 1}$, we get
\begin{align*}
  \sum_{\ell = 1}^{t - 1} Z_\ell X_\ell
  = \tilde{H}_t (\tilde{\theta}_t - \bar{\theta}_t)\,,
\end{align*}
where $Z_\ell \in \mathcal{N}(0, a^2)$ are i.i.d.\ normal random variables,
\begin{align*}
  \tilde{H}_t
  = \sum_{\ell = 1}^{t - 1} \dot{\mu}(X_\ell\T \theta'_t) X_\ell X_\ell\T\,,
\end{align*}
and $\theta'_t = \alpha \bar{\theta}_t + (1 - \alpha) \tilde{\theta}_t$ for some $\alpha \in [0, 1]$. Fix any $x \in \realset^d$ and let $U = x\T G_t^{-1} \sum_{\ell = 1}^{t - 1} Z_\ell X_\ell$. Then
\begin{align*}
  x\T G_t^{-1} \tilde{H}_t (\tilde{\theta}_t - \bar{\theta}_t)
  = U
  \sim \mathcal{N}(0, a^2 \normGt{x}^2)\,.
\end{align*}
Since $U$ is a normal random variable, we have that
\begin{align*}
  \probt{U \geq a \normGt{x}}
  & \geq 0.15\,, \\
  \probt{U \geq c \normGt{x}}
  & \leq \exp\left[- \frac{c^2}{2 a^2}\right]\,,
\end{align*}
for any $c > 0$.

Since all feature vectors have at most one non-zero entry, $G_t^{-1}$ and $\tilde{H}_t$ are diagonal, as is $G_t^{-1} \tilde{H}_t$. By the definitions of $G_t$ and $\tilde{H}_t$, diagonal entries of $G_t^{-1} \tilde{H}_t$ are non-negative and at most $\dot{\mu}_{\max}$. Let $x$ have at most one non-zero entry. Then $x\T (\tilde{\theta}_t - \bar{\theta}_t)$ and $x\T G_t^{-1} \tilde{H}_t (\tilde{\theta}_t - \bar{\theta}_t)$ have the same sign, which we use to derive
\begin{align*}
  0.15
  & \leq \probt{U \geq a \normGt{x}} \\
  & \leq \probt{\dot{\mu}_{\max} \, x\T (\tilde{\theta}_t - \bar{\theta}_t) \geq
  a \normw{x}{G_t^{-1}}} \\
  & = \probt{x\T (\tilde{\theta}_t - \bar{\theta}_t) \geq
  a \dot{\mu}_{\max}^{-1} \normGt{x}}\,.
\end{align*}
For $a = c_1 \dot{\mu}_{\max}$ and $x = x_1$, we get that event $E_{3, t}$ in \eqref{eq:theta tilde is optimistic} occurs with probability at least $0.15$.

The diagonal entries of $G_t^{-1} \tilde{H}_t$ are non-negative, and also at least $\dot{\mu}_{\min}$ on events $E$ and $E'$. So, on event $E$,
\begin{align*}
  \exp\left[- \frac{c^2}{2 a^2}\right]
  & \geq \probt{U \geq c \normGt{x}} \\
  & \geq \probt{U \geq c \normGt{x}, \, E' \text{ occurs}} \\
  & \geq \probt{\dot{\mu}_{\min} \, x\T (\tilde{\theta}_t - \bar{\theta}_t) \geq
  c \normGt{x}} - \frac{1}{n} \\
  & = \probt{x\T (\tilde{\theta}_t - \bar{\theta}_t) \geq
  c \dot{\mu}_{\min}^{-1} \normGt{x}} - \frac{1}{n}\,.
\end{align*}
For $c = a \sqrt{2 \log(K n)}$, $x = x_i$, and by the union bound over all $K$ arms, we get that event $\bar{E}_{2, t}$ in \eqref{eq:theta tilde is close} occurs with probability at most $2 / n$.
\end{proof}

\subsection{Discussion}
\label{sec:discussion}

The regret of \glmtsl is $\tilde{O}(d \sqrt{n \log K})$ (\cref{thm:glm-tsl upper bound}). Up to the factor of $\sqrt{\log K}$, this matches the gap-free bounds of \glmucb \citep{filippi10parametric} and \ucbglm \citep{li17provably}. As in \citet{agrawal13thompson}, the key idea in our analysis is to achieve optimism by inflating the covariance matrix in \glmtsl by $a = O(\sqrt{d \log n})$. This setting is too conservative in practice. Thus, in \cref{sec:experiments}, we also experiment with $a = O(1)$, which is known to work well in practice \citep{chapelle11empirical,russo18tutorial}.

The regret of \glmfpl is $\tilde{O}(d \sqrt{n \log K})$ (\cref{thm:glm-fpl upper bound}). Although the bound scales with $K$, $d$, and $n$ similarly to that in \cref{thm:glm-tsl upper bound}, it is worse in constant factors. For instance, $c_2$ is additionally multiplied by $\sqrt{\dot{\mu}_{\min}^{-1} \, \dot{\mu}_{\max}}$. The number of initial exploration rounds is also higher, since we need to guarantee that $\tilde{\theta}_t$ and $\theta_\ast$ are close with a high probability given any $\cF_{t - 1}$. As in \glmtsl, the suggested value of $a = O(\sqrt{d \log n})$ is too conservative in practice. Thus, we also experiment with $a = O(1)$ in \cref{sec:experiments}.

The regret bound of \glmfpl is proved under the assumption that feature vectors have at most one non-zero entry. We need this assumption for the following reason. We establish in \cref{lem:glm-fpl} that
\begin{align*}
  U
  = x\T G_t^{-1} \tilde{H}_t (\tilde{\theta}_t - \bar{\theta}_t)
  \sim \mathcal{N}(0, a^2 \normGt{x}^2)\,.
\end{align*}
Since $a \normGt{x}$ is one standard deviation of $U$, event $U > a \normGt{x}$ is likely. But we need event $U' = x\T (\tilde{\theta}_t - \bar{\theta}_t) > a \normGt{x}$ to be likely. If $G_t^{-1}$ and $\tilde{H}_t$ have different eigenvectors, $U$ and $U'$ can have different signs, and it is hard to relate them due to potential rotations by $G_t^{-1} \tilde{H}_t$. Our assumption guarantees that the eigenvectors of $G_t^{-1}$ and $\tilde{H}_t$ are identical. We leave the elimination of this assumption for future work.

The initial exploration in \glmtsl and \glmfpl can be implemented as follows. Let $\set{v_i}_{i = 1}^d \subseteq \set{x_i}_{i = 1}^K$ be any basis in $\realset^d$ and $M = \sum_{i = 1}^d v_i\T v_i$. Then, to satisfy assumptions $\lambda_{\min}(G_\tau) \geq C$ in \cref{thm:glm-tsl upper bound,thm:glm-fpl upper bound}, each arm in the basis is pulled $C \lambda_{\min}^{-1}(M)$ times.


\section{Experiments}
\label{sec:experiments}

We conduct two sets of experiments. In \cref{sec:logistic bandit experiments}, we assess the empirical regret of \glmtsl and \glmfpl in logistic bandits. Because of its simplicity and generality, the perturbation mechanism in \glmfpl can be easily applied to more complex models. We assess it on contextual bandit problems with neural networks in \cref{sec:deep bandit experiments}.

\subsection{Logistic Bandit}
\label{sec:logistic bandit experiments}

The goal of this experiment is to show that our proposed algorithms perform well. We experiment with a \emph{logistic bandit}, a GLM bandit where $\mu(v) = 1 / (1 + \exp[- v])$ and $Y_{i, t} \sim \mathrm{Ber}(\mu(x_i\T \theta_\ast))$. The number of arms is $K = 100$. To avoid bias in choosing problem instances, we generate them randomly: the feature vector of arm $i$ is drawn uniformly at random from $[-1, 1]^d$ and the parameter vector is $\theta_\ast \sim \mathcal{N}(\mathbf{0}, 3 d^{-2} I_d)$, where $I_d$ is a $d \times d$ identity matrix. By design, $\var{x_i\T \theta_\ast} = 1$, and so $x_i\T \theta_\ast \in [-4, 4]$ with a high probability. We vary the number of features $d$ from $5$ to $20$. The horizon is $n = 50\,000$ rounds and our results are averaged over $100$ problem instances.

\begin{figure*}[t]
  \centering
  \includegraphics[width=2.22in]{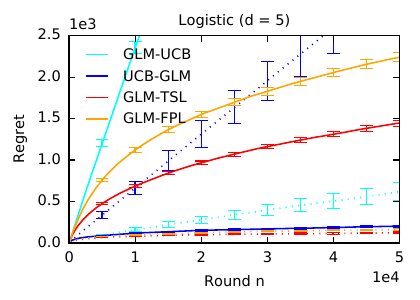}
  \includegraphics[width=2.22in]{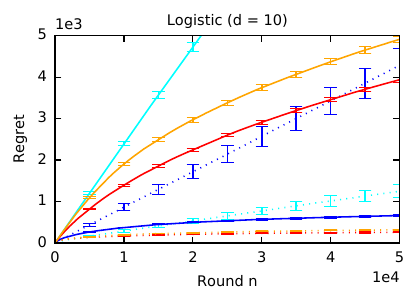}
  \includegraphics[width=2.22in]{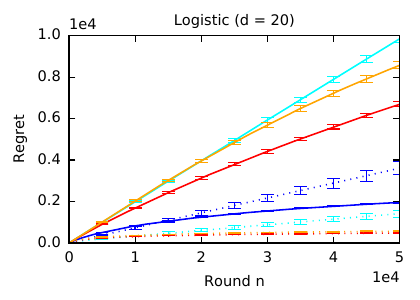}
  \vspace{-0.25in}
  \caption{Evaluation of \glmtsl and \glmfpl in logistic bandits. The $n$-round regret is shown as a function of $n$. The solid and dotted lines represented theory-suggested and informal designs, respectively.}
  \label{fig:logistic bandit}
\end{figure*}

Our baselines are two UCB algorithms, \glmucb \citep{filippi10parametric} and \ucbglm \citep{li17provably}. We experiment with two designs for each evaluated algorithm, \emph{theory} (as analyzed) and \emph{informal} (practical). For \glmtsl, we use $a$ from \cref{thm:glm-tsl upper bound} and $a = 1$, for which \eqref{eq:glm-tsl} reduces to sampling from the Laplace approximation. For \glmfpl, we use $a$ from \cref{thm:glm-fpl upper bound} and $a = 0.5$. We choose the latter since $a$ in \cref{thm:glm-fpl upper bound} is half that in \cref{thm:glm-tsl upper bound} in logistic models, since $\dot{\mu}_{\max} = 0.25$. We also implement $\glmucb$ and $\ucbglm$ with tighter confidence intervals, $0.5 \normw{x}{G^{-1}}$, where $x$ is the feature vector of the arm, $G$ is the sample covariance matrix, and $0.5$ is the maximum standard deviation of rewards in logistic models. All algorithms pull $d$ linearly independent arms initially and $\dot{\mu}_{\min}$ is set to the most optimistic value of $0.25$.

Our results are shown in \cref{fig:logistic bandit}. We observe that theory \glmtsl and \glmfpl outperform theory \glmucb, but not theory \ucbglm. The latter is known from prior algorithm designs. In particular, when \lints \citep{agrawal13thompson} is implemented as analyzed, it fails to outperform \linucb \citep{abbasi-yadkori11improved}; but it does outperform it when the theory-suggested posterior scaling is relaxed. This is indeed how \lints is usually implemented. Informal \glmucb and \ucbglm fail, and have linear regret in $n$. On the other hand, informal \glmtsl and \glmfpl have low regret, sublinear in $n$. We conclude that \glmtsl and \glmfpl have state-of-the-art performance in logistic bandits.

\subsection{Deep Bandit}
\label{sec:deep bandit experiments}

The second experiment is on contextual bandit problems, which are generated as follows. We fix a supervised learning dataset $\mathcal{D}$ and a target label $c$. The examples with label $c$ have random rewards $\mathrm{Ber}(0.75)$ while the other examples have random rewards $\mathrm{Ber}(0.25)$. In round $t$, the agent is presented $K = 10$ random examples $x_{i, t}$ from $\mathcal{D}$, which are arms. The agent learns a single generalization model that maps feature vector $x_{i, t}$ to its expected reward. The goal of the agent is to learn a good mapping quickly. Since our generalization models are imperfect, our evaluation metric is the \emph{average per-round reward} in $n$ rounds, which we define as $\sum_{t = 1}^n Y_t / n$.

\begin{figure*}[t]
  \centering
  \includegraphics[width=2.2in]{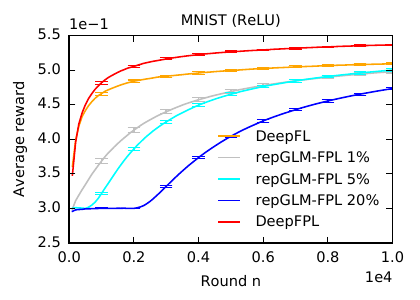}
  \includegraphics[width=2.2in]{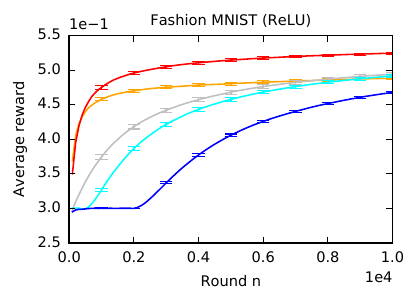}
  \includegraphics[width=2.2in]{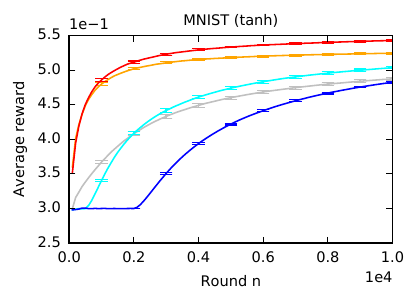}
  \vspace{-0.1in}
  \caption{Evaluation of \deepfpl on contextual bandit problems in \cref{sec:deep bandit experiments}.}
  \label{fig:deep bandit}
\end{figure*}

We experiment with two large-scale datasets: MNIST and Fashion MNIST. \emph{MNIST} \citep{lecun98gradientbased} is a dataset of $60$ thousand $28 \times 28$ gray-scale images of handwritten digits, from $0$ to $9$. \emph{Fashion MNIST} \citep{xiao17fashionmnist} is a dataset of $60$ thousand $28 \times 28$ gray-scale images in $10$ fashion categories. We generate $500$ bandit instances for each dataset, $50$ for each class in that dataset. The horizon is $n = 10\,000$ rounds and we report the average reward over all instances in each dataset.

We implement \glmfpl with the neural network generalization in Keras \citep{chollet15keras}. The neural network has a single fully-connected hidden layer with $50$ units. We experiment with both ReLU and tanh activation functions in the hidden layer. The output layer is a sigmoid. In each round, the model is updated using the adaptive optimizer Adam \citep{kingma15adam}, where the learning rate is $0.001$ and the mini-batch contains $32$ most recent examples. These are default settings in Keras. Yogi \citep{zaheer18adaptive} could be used instead of Adam. The rewards of the training examples are perturbed with i.i.d.\ $\mathcal{N}(0, a^2)$ noise where $a = 1$. We call this algorithm \deepfpl.

We consider two baselines. The first is a follow-the-leader variant of \deepfpl where $a = 0$. We call it \deepfl. The second is a variant of \emph{Neural Linear}, the best method in a recent large empirical study \citep{riquelme18deep}. This approach learns a representation separately of the bandit problem and applies an existing bandit algorithm to it. We learn the representation in $m$ percent of initial rounds by exploring randomly. The representation is the same neural network as in \deepfpl. After learning, we chop its head off and use the rest to embed feature vectors. The bandit algorithm is \glmfpl and we call this combined approach \repglmfpl. We experiment with $m$ from $1\%$ to $20\%$.

Our results are reported in \cref{fig:deep bandit}. We observe three major trends. First, \deepfpl achieves high average rewards of at least $0.5$, which is close to the theoretical optimum $0.25 \, (1 / K)^K + 0.75 \, (1 - (1 / K)^K) \approx 0.576$ in both our problems. Second, \deepfpl outperforms \deepfl. This shows that exploration is beneficial, since the only difference between \deepfpl and \deepfl is that \deepfpl perturbs rewards to explore. Third, \deepfpl outperforms all variants of \repglmfpl. This shows that interleaving of representation learning and exploration is beneficial. Also note that the best setting of $m$ in \repglmfpl depends on the problem. For instance, at $n = 10\,000$ rounds, $1\%$ and $5\%$ exploration is comparable in the first two plots, while $5\%$ exploration is superior in the last plot. \deepfpl does not need any such tunable parameter.

This experiment shows that \glmfpl generalizes easily to complex models and works well.  While it does not have regret guarantees in these models, it should be of interest to practitioners.


\section{Related Work}
\label{sec:related work}

In the infinite arm setting, \citet{abeille17linear} proved that the regret of \glmtsl is $\tilde{O}(d^\frac{3}{2} \sqrt{n})$. We prove that it is $\tilde{O}(d \sqrt{n \log K})$ when the number of arms is $K$. This is an improvement of $\sqrt{d / \log K}$ in our setting. We also match the result of \citet{abeille17linear} in the infinite arm setting. Specifically, if the space of arms was discretized on an $\varepsilon$-grid, and this discretization would not change the order of the regret, the number of arms would be $K = \varepsilon^{- d}$ and $\sqrt{\log K} = \sqrt{d \log(1 / \varepsilon)}$. Our analysis is different from \citet{abeille17linear} and is more like that of \citet{agrawal13thompson}. We also match, up to the factor of $\sqrt{\log K}$, the bounds of most non-randomized GLM bandit algorithms \citep{filippi10parametric,zhang16online,li17provably,jun17scalable}, which are $\tilde{O}(d \sqrt{n})$.

\citet{dong19performance} proved that the $n$-round Bayes regret of \glmtsl is $\tilde{O}(d \sqrt{n})$. This bound is for a weaker performance metric than in this work, the Bayes regret; applies only to logistic bandits; and makes strong assumptions on the features of arms and $\theta_\ast$. However, it does not depend on $\dot{\mu}_{\min}$, which is a significant advance.

Similarly to \glmtsl, we prove that the regret of \glmfpl is $\tilde{O}(d \sqrt{n \log K})$. This regret bound is under the assumption that feature vectors have at most one non-zero entry. Although limited, this result is non-trivial since the number of potentially optimal arms is $2 d$, two per dimension. This is the first frequentist regret bound for exploration by Gaussian noise perturbations in a non-linear model. The good empirical performance of \glmfpl (\cref{sec:experiments}) suggests that the regret bound should hold in general, and we leave the more general analysis as future work.

\glmtsl is a variant of Thompson sampling. Thompson sampling \citep{thompson33likelihood,agrawal13further,russo18tutorial} is relatively well understood in linear bandits \citep{agrawal13thompson,valko14spectral}. However, it is difficult to extend it to non-linear problems because their posterior distributions are complex and have to be approximated. In general, posterior approximations in bandits are computationally costly and lack regret guarantees \citep{gopalan14thompson,kawale15efficient,lu17ensemble,riquelme18deep,lipton18bbq,liu18customized}. We provide guarantees in this work.

\glmfpl is a follow-the-perturbed-leader algorithm \citep{hannan57approximation,kalai05efficient}. We can also view it as randomized least-squares value iteration \citep{osband16generalization} applied to bandits. Our instance, additive Gaussian noise in a GLM, is novel. \glmfpl is also closely related to perturbed-history exploration \citep{kveton19garbage,kveton19perturbed,kveton19perturbed2}. \citet{kveton19perturbed2} proposed a logistic bandit algorithm that explores by perturbing its history with Bernoulli noise. This algorithm was not analyzed and is less general than \glmfpl, as it is only for logistic bandits.


\section{Conclusions}
\label{sec:conclusions}

We study two randomized algorithms for GLM bandits, \glmtsl and \glmfpl. The key idea in both algorithms is to explore by perturbing the maximum likelihood estimate in round $t$. We analyze \glmtsl and \glmfpl, and prove that their $n$-round regret is $\tilde{O}(d \sqrt{n \log K})$. Both \glmtsl and \glmfpl perform well empirically in logistic bandits. \glmfpl can be easily generalized to more complex problems. Our experiments with neural networks are very encouraging, and indicate that \glmfpl can be analyzed beyond GLM bandits. We plan to conduct such analyses in future work.

Our analysis is under the assumption that the feature vectors of arms are fixed and do not change over time. This assumption can be lifted. The only part of the proof that changes is that the number of initial exploration rounds $\tau$ after which $\lambda_{\min}(G_\tau)$ (\cref{thm:glm-tsl upper bound,thm:glm-fpl upper bound}) is large enough becomes a random variable. \citet{li17provably} analyzed this random variable and we can directly reuse their result.

\clearpage

\bibliographystyle{plainnat}
\bibliography{References}

\clearpage
\onecolumn
\appendix


\section{Regret Bounds}
\label{sec:regret bounds}

The following lemma bounds the expected per-round regret of any randomized algorithm that chooses the perturbed solution in round $t$, $\tilde{\theta}_t$, as a function of the history.

\perroundregret*
\begin{proof}
Let $\tilde{\Delta}_i = x_1\T \theta_\ast - x_i\T \theta_\ast$ and $c = c_1 + c_2$. Let
\begin{align*}
  \bar{S}_t
  = \set{i \in [K]: c \normGt{x_i} \geq \tilde{\Delta}_i}
\end{align*}
be the set of \emph{undersampled arms} in round $t$. Note that $1 \in \bar{S}_t$ by definition. We define the set of \emph{sufficiently sampled arms} as $S_t = [K] \setminus \bar{S}_t$. Let $J_t = \argmin_{i \in \bar{S}_t} \normGt{x_i}$ be the \emph{least uncertain undersampled arm} in round $t$.

In all steps below, we assume that event $E_{1, t}$ occurs. In round $t$ on event $E_{2, t}$,
\begin{align*}
  \Delta_{I_t}
  & \leq \dot{\mu}_{\max} \, \tilde{\Delta}_{I_t}
  = \dot{\mu}_{\max} \left(\tilde{\Delta}_{J_t} +
  x_{J_t}\T \theta_\ast - x_{I_t}\T \theta_\ast\right)
  \leq \dot{\mu}_{\max} \left(\tilde{\Delta}_{J_t} +
  x_{J_t}\T \tilde{\theta}_t - x_{I_t}\T \tilde{\theta}_t +
  c \, (\normGt{x_{I_t}} + \normGt{x_{J_t}})\right) \\
  & \leq \dot{\mu}_{\max} \, c \left(\normGt{x_{I_t}} + 2 \normGt{x_{J_t}}\right)\,,
\end{align*}
where the first inequality holds because $\dot{\mu}_{\max}$ is the maximum derivative of $\mu$, the second is by the definitions of events $E_{1, t}$ and $E_{2, t}$, and the last follows from the definitions of $I_t$ and $J_t$. Now we take the expectation of both sides and get
\begin{align*}
  \Et{\Delta_{I_t}}
  = \Et{\Delta_{I_t} \I{E_{2, t}}} + \Et{\Delta_{I_t} \I{\bar{E}_{2, t}}}
  \leq \dot{\mu}_{\max} \, c \, \Et{\normGt{x_{I_t}} + 2 \normGt{x_{J_t}}} +
  \Delta_{\max} \, \probt{\bar{E}_{2, t}}\,.
\end{align*}
The last step is to replace $\Et{\normGt{x_{J_t}}}$ with $\Et{\normGt{x_{I_t}}}$. To do so, observe that
\begin{align*}
  \Et{\normGt{x_{I_t}}}
  \geq \Et{\normGt{x_{I_t}} \,\middle|\, I_t \in \bar{S}_t}
  \probt{I_t \in \bar{S}_t}
  \geq \normGt{x_{J_t}} \, \probt{I_t \in \bar{S}_t}\,,
\end{align*}
where the last inequality follows from the definition of $J_t$ and that $\bar{S}_t$ is $\cF_{t - 1}$-measurable. We rearrange the inequality as $\normGt{x_{J_t}} \leq \Et{\normGt{x_{I_t}}} / \ \probt{I_t \in \bar{S}_t}$ and bound $\probt{I_t \in \bar{S}_t}$ from below next.

In particular, on event $E_{1, t}$,
\begin{align*}
  \probt{I_t \in \bar{S}_t}
  & \geq \probt{\exists i \in \bar{S}_t: x_i\T \tilde{\theta}_t >
  \max_{j \in S_t} x_j\T \tilde{\theta}_t}
  \geq \probt{x_1\T \tilde{\theta}_t >
  \max_{j \in S_t} x_j\T \tilde{\theta}_t} \\
  & \geq \probt{x_1\T \tilde{\theta}_t >
  \max_{j \in S_t} x_j\T \tilde{\theta}_t, \, E_{2, t} \text{ occurs}}
  \geq \probt{x_1\T \tilde{\theta}_t >
  x_1\T \theta_\ast, \, E_{2, t} \text{ occurs}} \\
  & \geq \probt{x_1\T \tilde{\theta}_t > x_1\T \theta_\ast} -
  \probt{\bar{E}_{2, t}}
  \geq \probt{x_1\T \tilde{\theta}_t - x_1\T \bar{\theta}_t > c_1 \normGt{x_1}} -
  \probt{\bar{E}_{2, t}}\,.
\end{align*}
Note that we require a sharp inequality because $I_t \in \bar{S}_t$ is not guaranteed on event $\set{\exists i \in \bar{S}_t: x_i\T \tilde{\theta}_t \geq \max_{j \in S_t} x_j\T \tilde{\theta}_t}$. The fourth inequality holds because on event $E_{1, t} \cap E_{2, t}$,
\begin{align*}
  x_j\T \tilde{\theta}_t
  \leq x_j\T \theta_\ast + c \normGt{x_j}
  < x_j\T \theta_\ast + \tilde{\Delta}_j
  = x_1\T \theta_\ast
\end{align*}
holds for any $j \in S_t$. The last inequality holds because $x_1\T \theta_\ast \leq x_1\T \bar{\theta}_t + c_1 \normGt{x_1}$ holds on event $E_{1, t}$. Finally, we use the definitions of $p_2$ and $p_3$ to complete the proof.
\end{proof}

The regret bound of \glmtsl is proved below.

\glmtslregretbound*
\begin{proof}
Fix $\tau \in [n]$. Let
\begin{align*}
  E_{4, t}
  = \set{\normw{\bar{\theta}_t - \theta_\ast}{2} \leq 1}
\end{align*}
and $p_4 \geq \prob{\bar{E}_{4, t}}$ for $t \geq \tau$. Let $p_1 \geq \prob{\bar{E}_{1, t}, E_{4, t}}$, $p_2 \geq \probt{\bar{E}_{2, t}}$ on event $E_{4, t}$, and $p_3 \leq \probt{E_{3, t}}$. By elementary algebra, we get
\begin{align*}
  R(n)
  & \leq \sum_{t = \tau}^n \E{\Delta_{I_t}} + \tau \Delta_{\max} \\
  & \leq \sum_{t = \tau}^n \E{\Delta_{I_t} \I{E_{4, t}}} +
  (\tau + p_4 n) \Delta_{\max} \\
  & \leq \sum_{t = \tau}^n \E{\Delta_{I_t} \I{E_{1, t}, E_{4, t}}} +
  (\tau + (p_1 + p_4) n) \Delta_{\max} \\
  & = \sum_{t = \tau}^n \E{\Et{\Delta_{I_t}} \I{E_{1, t}, E_{4, t}}} +
  (\tau + (p_1 + p_4) n) \Delta_{\max}\,.
\end{align*}
To get $p_1 \leq 1 / n$, we set $c_1$ as in \cref{lem:theta bar concentration}. Now we apply \cref{lem:per-round regret} to $\Et{\Delta_{I_t}} \I{E_{1, t}, E_{4, t}}$ and get
\begin{align*}
  R(n)
  & \leq \dot{\mu}_{\max} (c_1 + c_2) \left(1 + \frac{2}{p_3 - p_2}\right)
  \E{\sum_{t = \tau}^n \normw{x_{I_t}}{G_t^{-1}}} +
  (\tau + (p_1 + p_2 + p_4) n) \Delta_{\max}\,,
\end{align*}
where $a$ and $c_2$ are set as in \cref{lem:glm-tsl}. For these settings, $p_2 \leq 1 / n$ and $p_3 \geq 0.15$. To bound $\sum_{t = \tau}^n \normw{x_{I_t}}{G_t^{-1}}$, we use Lemma 2 in \citet{li17provably}. Finally, to get $p_4 \leq 1 / n$, we choose $\tau$ as in \cref{lem:glm-tsl unit ball}.
\end{proof}

The regret bound of \glmfpl is proved below.

\glmfplregretbound*
\begin{proof}
The proof is almost identical to that of \cref{thm:glm-tsl upper bound}. There are two main differences. First, $a$ and $c_2$ are set as in \cref{lem:glm-fpl}. For these settings, $p_2 \leq 2 / n$ and $p_3 \geq 0.15$. Second, $\tau$ is set as in \cref{lem:glm-fpl unit ball}.
\end{proof}


\section{Technical Lemmas}
\label{sec:technical lemmas}

We need an extension of Theorem 1 in \citet{abbasi-yadkori11improved}, which is concerned with concentration of a certain vector-valued martingale. The setup of the claim is as follows. Let $(\cF_t)_{t \geq 0}$ be a filtration, $(\eta_t)_{t \geq 1}$ be a stochastic process such that $\eta_t$ is real-valued and $\cF_t$-measurable, and $(X_t)_{t \geq 1}$ be another stochastic process such that $X_t$ is $\realset^d$-valued and $\cF_{t - 1}$-measurable. We also assume that $(\eta_t)_t$ is conditionally $R^2$-sub-Gaussian, that is
\begin{align}
  \forall \lambda \in \realset: \quad
  \condE{\exp[\lambda \eta_t]}{\cF_{t - 1}}
  \leq \exp\left[\frac{\lambda^2 R^2}{2}\right]\,.
\end{align}
We call the triplet $((X_t)_t, (\eta_t)_t, \mathbb{F})$ \say{nice} when these conditions hold. The modified claim is stated and proved below.

\begin{lemma}
\label{lem:self-normalized bound} Let $((X_t)_t, (\eta_t)_t, \mathbb{F})$ be a \say{nice} triplet, $S_t = \sum_{s = 1}^t \eta_s X_s$, $V_t = \sum_{s = 1}^t X_s X_s\T$; and for $V \succ 0$, let $\tau_0 = \min \set{t \geq 1: V_t \succeq V}$. Then for any $\delta \in (0, 1)$ and $\mathbb{F}$-stopping time $\tau \geq 1$ such that $\tau \geq \tau_0$ holds almost surely, with probability at least $1 - \delta$,
\begin{align*}
  \normw{S_\tau}{V_\tau^{-1}}^2
  \leq 2 R^2 \log\left(\frac{\det(V_\tau)^\frac{1}{2}
  \det(V_{\tau_0})^{- \frac{1}{2}}}{\delta}\right)\,.
\end{align*}
\end{lemma}
\begin{proof}
The proof in \citet{abbasi-yadkori11improved} can easily adjusted as follows. If $((X_t)_t, (\eta_t)_t, \mathbb{F})$ is a \say{nice} triplet, then for any $\delta \in (0, 1)$, $\cF_0$-measurable matrix $V \succ 0$, and stopping time $\tau \geq 1$,
\begin{align}
  \condprob{\normw{S_\tau}{V_\tau^{-1}}^2
  \leq 2 R^2 \log\left(\frac{\det(V_\tau)^\frac{1}{2}
  \det(V_{\tau_0})^{- \frac{1}{2}}}{\delta}\right)}{\cF_0}
  \geq 1 - \delta\,.
  \label{eq:nice triplet}
\end{align}
Now, for $t \geq 0$, let $X_t' = X_{\tau_0 + t}$, $\eta_t' = \eta_{\tau_0 + t}$, and $\cF_t' = \cF_{\tau_0 + t}$. Then $((X_t')_{t \geq 1}$, $(\eta_t')_{t \geq 1}$, $(\cF_t')_{t \geq 0})$ is a nice triplet and the result follows from \eqref{eq:nice triplet}.
\end{proof}

We use the last lemma to prove the following result.

\begin{lemma}
\label{lem:theta bar concentration} Let $c_1 = \sigma \dot{\mu}_{\min}^{-1} \sqrt{d \log(n / d) + 2 \log n}$ and $\tau$ be any round such that $\lambda_{\min}(G_\tau) \geq 1$. Then for any $t \geq \tau$,
\begin{align*}
  \prob{\bar{E}_{1, t} \emph{ occurs}, \,
  \normw{\bar{\theta}_t - \theta_\ast}{2} \leq 1}
  \leq 1 / n\,.
\end{align*}
\end{lemma}
\begin{proof}
Let $S_t = \sum_{\ell = 1}^{t - 1} (Y_\ell - \mu(X_\ell\T \theta_\ast)) X_\ell$. By \cref{lem:workhorse}, where $\cD_1 = \set{(X_\ell, \mu(X_\ell\T \theta_\ast))}_{\ell = 1}^{t - 1}$ and $\cD_2 = \set{(X_\ell, Y_\ell)}_{\ell = 1}^{t - 1}$, we have that
\begin{align*}
  S_t
  = \underbrace{\nabla^2 L(\cD_1; \theta')}_V (\bar{\theta}_t - \theta_\ast)\,,
\end{align*}
where $\theta' = \alpha \theta_\ast + (1 - \alpha) \bar{\theta}_t$ for $\alpha \in [0, 1]$. We rearrange the equality as $V^{-1} S_t = \bar{\theta}_t - \theta_\ast$ and note that $\dot{\mu}_{\min} G_t \preceq V$ on $\normw{\bar{\theta}_t - \theta_\ast}{2} \leq 1$. Now fix arm $i$. By the Cauchy-Schwarz inequality and from the above discussion,
\begin{align*}
  \abs{x_i\T \bar{\theta}_t - x_i\T \theta_\ast}
  & \leq \normw{\bar{\theta}_t - \theta_\ast}{G_t} \normw{x_i}{G_t^{-1}}
  = (\bar{\theta}_t - \theta_\ast)\T G_t (\bar{\theta}_t - \theta_\ast)
  \normw{x_i}{G_t^{-1}} \\
  & = S_t\T V^{-1} G_t V^{-1} S_t \normw{x_i}{G_t^{-1}}
  \leq \dot{\mu}_{\min}^{-2} \normw{S_t}{G_t^{-1}} \normw{x_i}{G_t^{-1}}\,.
\end{align*}
By \eqref{eq:s_t upper bound} in \cref{lem:glm-tsl unit ball}, which is derived using \cref{lem:self-normalized bound}, $\normw{S_t}{G_t^{-1}} \leq \sigma \sqrt{d \log(n / d) + 2 \log n}$ holds with probability at least $1 - 1 / n$ in any round $t \geq \tau$. In this case, event $E_{1, t}$ is guaranteed to occur when $c_1$ is set as in the claim. It follows that $\bar{E}_{1, t}$ occurs on $\normw{\bar{\theta}_t - \theta_\ast}{2} \leq 1$ with probability of at most $1 / n$.
\end{proof}

The number of initial exploration rounds in \glmtsl is set below.

\begin{lemma}
\label{lem:glm-tsl unit ball} Let $\tau$ be any round such that
\begin{align*}
  \lambda_{\min}(G_\tau)
  \geq \max \set{\sigma^2 \dot{\mu}_{\min}^{-2} (d \log(n / d) + 2 \log n), \, 1}\,.
\end{align*}
Then for any $t \geq \tau$, $\prob{\normw{\bar{\theta}_t - \theta_\ast}{2} > 1} \leq 1 / n$.
\end{lemma}
\begin{proof}
Fix round $t$ and let $S_t = \sum_{\ell = 1}^{t - 1} (Y_\ell - \mu(X_\ell\T \theta_\ast)) X_\ell$. By the same argument as in the proof of Theorem 1 in \citet{li17provably}, who use Lemma A of \citet{chen99strong}, we have that
\begin{align*}
  \normw{S_t}{G_t^{-1}}
  \leq \dot{\mu}_{\min} \sqrt{\lambda_{\min}(G_t)}
  \implies \normw{\bar{\theta}_t - \theta_\ast}{2}
  \leq 1
\end{align*}
Now fix $\tau$ such that $\lambda_{\min}(G_\tau) \geq 1$. For any $t \geq \tau$, $G_t \succeq G_\tau$ and thus
\begin{align}
  \normw{S_t}{G_t^{-1}}
  \leq \dot{\mu}_{\min} \sqrt{\lambda_{\min}(G_\tau)}
  \implies \normw{\bar{\theta}_t - \theta_\ast}{2}
  \leq 1\,.
  \label{eq:inside unit ball}
\end{align}
In the next step, we bound $\normw{S_t}{G_t^{-1}}$ from above. By \cref{lem:self-normalized bound},
\begin{align*}
  \normw{S_t}{G_t^{-1}}^2
  \leq 2 \sigma^2 \log(\det(G_t)^\frac{1}{2} \det(G_\tau)^{- \frac{1}{2}} n)
\end{align*}
holds jointly in all rounds $t \geq \tau$ with probability at least $1 - 1 / n$. By Lemma 11 in \citet{abbasi-yadkori11improved} and from $\normw{X_t}{2} \leq 1$, we get $\log\det(G_t) \leq d \log(n / d)$. By the choice of $\tau$, $\det(G_\tau)^{-1} \leq 1$. It follows that
\begin{align}
  \normw{S_t}{G_t^{-1}}^2
  \leq \sigma^2 (d \log(n / d) + 2 \log n)
  \label{eq:s_t upper bound}
\end{align}
for any $t \geq \tau$ with probability at least $1 - 1 / n$. Now we combine this claim with \eqref{eq:inside unit ball} and have that $\normw{\bar{\theta}_t - \theta_\ast}{2} \leq 1$ holds with probability at least $1 - 1 / n$ when
\begin{align*}
  \lambda_{\min}(G_\tau)
  \geq \sigma^2 \dot{\mu}_{\min}^{-2} (d \log(n / d) + 2 \log n)\,.
\end{align*}
This concludes the proof.
\end{proof}

The number of initial exploration rounds in \glmfpl is set below.

\begin{lemma}
\label{lem:glm-fpl unit ball} Let $\tau$ be any round such that
\begin{align*}
  \lambda_{\min}(G_\tau)
  \geq \max \set{4 \sigma^2 \dot{\mu}_{\min}^{-2} (d \log(n / d) + 2 \log n), \,
  8 a^2 \dot{\mu}_{\min}^{-2} \log n, \, 1}\,.
\end{align*}
Then for any $t \geq \tau$, $\prob{\normw{\bar{\theta}_t - \theta_\ast}{2} > 1 / 2} \leq 1 / n$ and $\probt{\normw{\tilde{\theta}_t - \theta_\ast}{2} > 1} \leq 1 / n$ on event $\normw{\bar{\theta}_t - \theta_\ast}{2} \leq 1 / 2$.
\end{lemma}
\begin{proof}
Fix round $t$. Let $S_t$ be defined as in \cref{lem:glm-tsl unit ball} and $\tau_1$ be any round such that
\begin{align*}
  \lambda_{\min}(G_{\tau_1})
  \geq \min \set{4 \sigma^2 \dot{\mu}_{\min}^{-2} (d \log(n / d) + 2 \log n), \, 1}\,.
\end{align*}
Then by the same argument as in \cref{lem:glm-tsl unit ball}, $\prob{\normw{\bar{\theta}_t - \theta_\ast}{2} > 1 / 2} \leq 1 / n$ holds for any $t \geq \tau_1$.

Now fix round $t$, history $\cF_{t - 1}$, and assume that $\normw{\bar{\theta}_t - \theta_\ast}{2} \leq 1 / 2$ holds. Let
\begin{align*}
  \bar{S}_t
  = \sum_{\ell = 1}^{t - 1} (Y_\ell + Z_\ell - \mu(X_\ell\T \bar{\theta}_t)) X_\ell
  = \sum_{\ell = 1}^{t - 1} Z_\ell X_\ell\,,
\end{align*}
where the last equality holds because $\sum_{\ell = 1}^{t - 1} (Y_\ell - \mu(X_\ell\T \bar{\theta}_t)) X_\ell = \mathbf{0}$. Since $\normw{\bar{\theta}_t - \theta_\ast}{2} \leq 1 / 2$, the $0.5$-ball centered at $\bar{\theta}_t$ is within the unit ball centered at $\theta_\ast$. So, the minimum derivative of $\mu$ in the $0.5$-ball is not larger than that in the unit ball, and we have by a similar argument to \cref{lem:glm-tsl unit ball} that
\begin{align}
  \normw{\bar{S}_t}{G_t^{-1}}
  \leq \frac{1}{2} \dot{\mu}_{\min} \sqrt{\lambda_{\min}(G_t)}
  \implies \normw{\tilde{\theta}_t - \bar{\theta}_t}{2} \leq \frac{1}{2}\,.
  \label{eq:inside half ball}
\end{align}
By definition, $\normw{\bar{S}_t}{G_t^{-1}} = \normw{U}{2}$ for $U = G_t^{- \frac{1}{2}} \sum_{\ell = 1}^{t - 1} Z_\ell X_\ell$. Since $Z_\ell$ are i.i.d.\ random variables that are resampled in each round, we have $U \sim \mathcal{N}(\mathbf{0}, a^2 I_d)$ given $\cF_{t - 1}$, and that $\normw{U}{2} \leq a \sqrt{2 \log n}$ holds with probability at least $1 - 1 / n$ given $\cF_{t - 1}$. Now we combine this claim with \eqref{eq:inside half ball} and have that $\normw{\tilde{\theta}_t - \bar{\theta}_t}{2} \leq 1 / 2$ holds with probability at least $1 - 1 / n$ for any round $t$ such that
\begin{align*}
  \lambda_{\min}(G_t)
  \geq 8 a^2 \dot{\mu}_{\min}^{-2} \log n.
\end{align*}
For any such round, when $\normw{\bar{\theta}_t - \theta_\ast}{2} \leq 1 / 2$ holds, $\probt{\normw{\tilde{\theta}_t - \theta_\ast}{2} \leq 1} \geq 1 - 1 / n$. This concludes our proof.
\end{proof}

\end{document}